\documentclass[conference]{IEEEtran}
\usepackage{xcolor}
\usepackage{url}
\usepackage{graphicx}
\usepackage{graphics}
\usepackage{amsmath}
\usepackage{amssymb}
\usepackage{amsthm}
\usepackage{bm}
\usepackage{algorithm}
\usepackage{algpseudocode}
\usepackage{booktabs}
\usepackage{tikz}
\usepackage{epsfig}
\usepackage{times}
\usetikzlibrary{calc,arrows.meta,angles,quotes,positioning,patterns}

\newcommand\copyrightnotice{%
\begin{tikzpicture}[remember picture,overlay]
  \node[anchor=north, yshift=-15pt] at (current page.north) {%
    \parbox{0.85\textwidth}{\centering\footnotesize
      This work has been submitted to the IEEE for possible publication.
      Copyright may be transferred without notice, after which this
      version may no longer be accessible.}};
\end{tikzpicture}}

\newtheorem{lemma}{Lemma}
\newtheorem{proposition}{Proposition}
\newtheorem{remark}{Remark}

\IEEEoverridecommandlockouts                              % This command is only needed if
\title{\LARGE \bf
Observability Analysis and Online Calibration of Visual-Inertial-Wheel Odometry for 4WIS4WID Mobile Robots
}

\author{Branimir Ćaran$^{1}$, Vladimir Milić$^{1}$,  Bojan Šekoranja$^{1}$, Bojan Jerbić$^{1,2}$% <-this % stops a space
\thanks{$^{1}$Authors are with Faculty of Mechanical Engineering and Naval Arhitecture, University of Zagreb, Zagreb, 10000, Croatia
        {\tt\small branimir.caran@fsb.unizg.hr}}
\thanks{$^{2}$ Bojan Jerbic is with Croatian Academy of Sciences and Arts, Zagreb, 10000, Croatia}
}

\begin{document}

\maketitle
\copyrightnotice
\thispagestyle{empty}
\pagestyle{empty}

%%%%%%%%%%%%%%%%%%%%%%%%%%%%%%%%%%%%%%%%%%%%%%%%%%%%%%%%%%%%%%%%%%%%%%%%%%%%%%%%
\begin{abstract}

In this paper, we present a visual-inertial-wheel odometry (VIWO) framework with online calibration for four-wheel independently steered and driven (4WIS4WID) mobile robots. We derive a 2D odometry model directly from the four driving velocities and steering angles, using both the longitudinal rolling constraints and the lateral no-slip constraints of all wheels. A preintegration model and analytical Jacobians are developed for efficient filtering and calibration. An observability analysis of the linearized VIWO system shows that a drive-only model makes all steering offsets unobservable, whereas the proposed redundant model restores their observability. The analysis also identifies four standard VINS unobservable directions and three additional directions associated with the arbitrary placement of the odometry reference frame. Furthermore, we characterize several degenerate motions, including zero yaw rate, constant steering, and a non-rolling wheel, and derive the corresponding excitation conditions for the thirteen wheel intrinsics that remain after fixing the odometry frame reference. The performance of the proposed system has been demonstrated in both simulation and real-world experiments on a 4WIS4WID mobile robot.
\end{abstract}

%%%%%%%%%%%%%%%%%%%%%%%%%%%%%%%%%%%%%%%%%%%%%%%%%%%%%%%%%%%%%%%%%%%%%%%%%%%%%%%%
\section{INTRODUCTION}
\label{sec:introduction}

Mobile robots with 4WIS4WID kinematics combine the terrain capabilities of conventional wheels with
omnidirectional planar motion. They have found applications in autonomous
vehicles~\cite{Jeong2022}, service robotics~\cite{Chen2025},
logistics~\cite{Hu2025} and inspection~\cite{Bozic2022}. Unlike
differential drive, Ackermann, and skid-steering platforms, independent
steering allows the instantaneous center of rotation (ICR) to be placed
arbitrarily in the plane~\cite{Lam2009}, enabling lateral motion, tight
turning, and continuous reconfiguration of the wheel constraint
directions~\cite{Muzuhno2024}. Fig.~\ref{fig:kinematics} shows the
resulting wheel geometry and the per-wheel decomposition of the
contact-point velocity.

\begin{figure}[t]
  \centering
  \includegraphics[width=252.0pt]{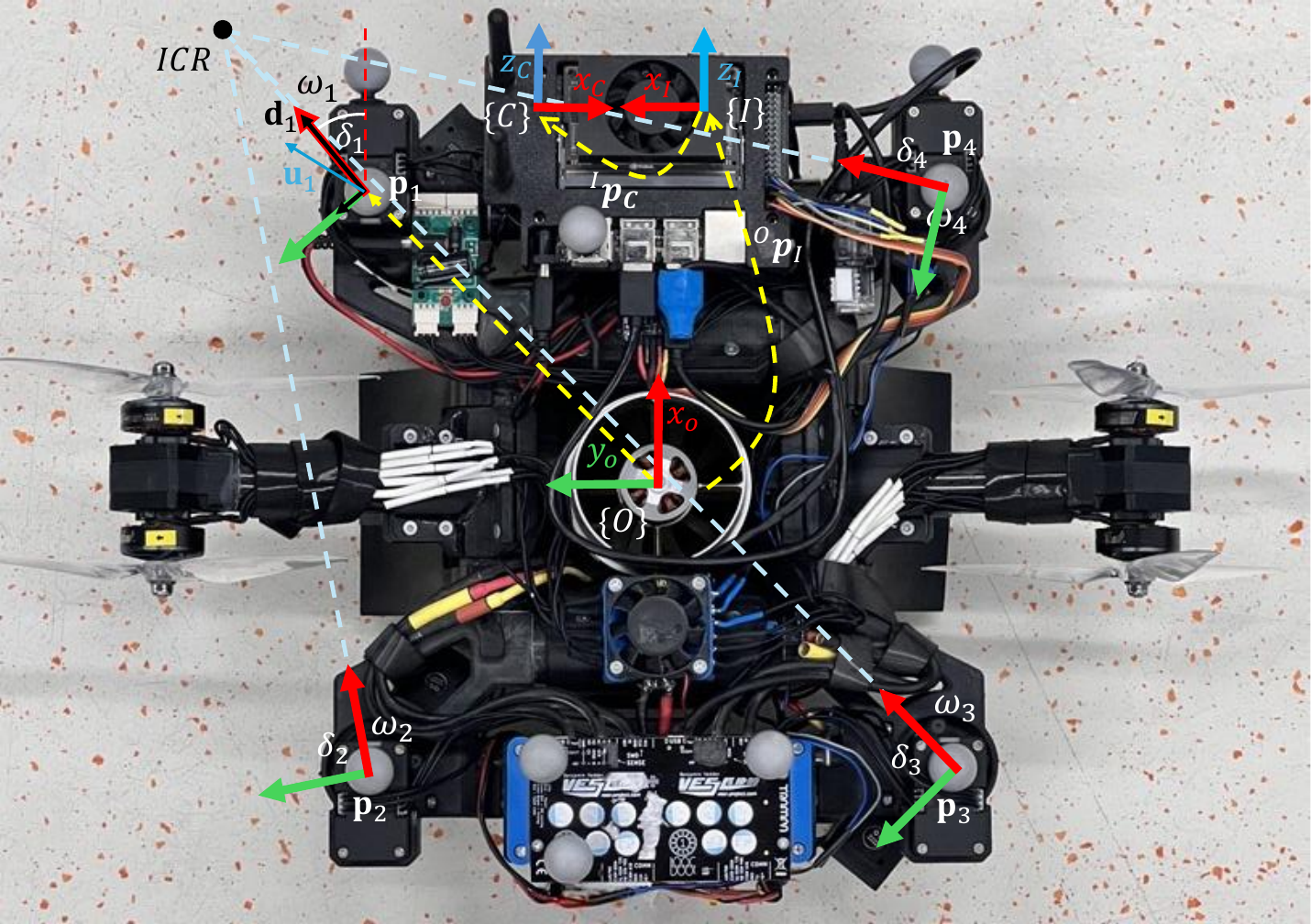}
  \caption{4WIS4WID mobile robot where each wheel $i$ is
  independently steered by $\delta_i$ and driven at $\omega_{d_i}$, so the
ICR may be placed anywhere in
  the plane.}
  \label{fig:kinematics}
\end{figure}

This maneuverability comes at the cost of substantially increased kinematic
complexity. The considered 4WIS4WID model contains sixteen intrinsic
parameters: four wheel radii, eight wheel positions and four
steering angle offsets~\cite{Lee2015JEng}. This is considerably more than in
commonly studied wheeled platforms such as differential drive, Ackermann, and
skid-steering robots~\cite{Siegwart2011} and introduces additional
calibration and observability challenges when wheel measurements are tightly
coupled with visual and inertial sensing.

Wheel encoders are commonly fused with visual-inertial odometry (VIO) to
improve state estimation for ground robots. Wu et al.~\cite{Wu2017}
incorporated wheel measurements together with planar and non-holonomic
constraints into a sliding-window estimator, showing that proprioceptive
odometry can resolve the scale ambiguity arising under constant acceleration
motion. Zhang et al.~\cite{Zhang2021TRO} formulated the motion manifold of
non-holonomic ground robots parametrically, enabling 6 DoF
state estimation from inherently planar wheel measurements.
Optimization based VIWO formulations have
also been proposed~\cite{Liu2024}, incorporating wheel measurements into
sliding-window optimization with initialization and extrinsic calibration.

These approaches generally assume a fixed and known wheel kinematic model.
Since wheel parameters may vary due to tire wear, load changes, pressure
variation, and mechanical inaccuracies, several works have considered their
online estimation together with the navigation state.
Martinelli~\cite{Martinelli2011} analysed the observability of wheel intrinsics
and planar sensor extrinsics and identified degenerate motion conditions,
while Censi et al.~\cite{Censi2013} studied simultaneous calibration of wheel
odometry and sensor parameters and characterized motions under which these
parameters are observable.

Lee et al.~\cite{Lee2020IROS} derive a 2D wheel odometry measurement model directly from raw
encoder readings, preintegrate it between camera clones, and jointly estimate
the wheel radii, baseline length, and odometry-IMU spatiotemporal extrinsics
within a Multi-State Constraint Kalman Filter (MSCKF). Their observability
analysis identifies several degenerate motions for the calibration
parameters. Related formulations have considered different kinematic
structures and sensing modalities. Zuo et al.~\cite{Zuo2024TASE} model
skid-steering platforms through ICR and estimate
time and terrain varying ICR parameters online, while tightly coupled
LiDAR-IMU-wheel formulations with online kinematic calibration have also
been investigated~\cite{Okawara2024}.

Observability analysis determines which states and calibration parameters can
be recovered from the available measurements and is fundamental to consistent
estimator design~\cite{Hesch2014,Huang2012}. For VIO systems,
Hesch et al.~\cite{Hesch2014} established the four standard unobservable
directions corresponding to global position and rotation about gravity.
Yang et al.~\cite{Yang2023TRO} extended this analysis to full sensor
self-calibration, showing that VINS with complete calibration of IMU and camera
intrinsics and spatiotemporal extrinsics retains these four unobservable
directions under sufficiently exciting motion. Related observability analyses
have been presented for aided inertial navigation with heterogeneous geometric
features~\cite{Yang2019TRO}, differential drive VIWO~\cite{Lee2020IROS},
skid-steering systems~\cite{Zuo2024TASE}, and optimization-based
formulations~\cite{Liu2024}.

Existing wheel-odometry calibration and observability analyses primarily consider differential drive, Ackermann, or skid-steering platforms,
whose ICR is constrained by the kinematic structure with a small number of parameters. 4WIS4WID platforms introduce sixteen intrinsic parameters and allow the ICR to move freely throughout the plane. Although 4WIS4WID robots have been extensively studied from the perspective of kinematics, dynamics, and control~\cite{Wang2026}, their online calibration and observability properties remain largely unexplored. Prior work on these robots~\cite{Caran2025ECMR} presented a loosely coupled pose estimator fusing wheel odometry, visual odometry, and inertial measurements, with the kinematic parameters calibrated offline and subsequently held fixed. To the best of the authors' knowledge, online calibration and observability analysis for this platform have not yet been addressed.

In this work, we develop a tightly coupled VIWO framework with online
calibration for 4WIS4WID mobile robots. The camera-IMU calibration and
odometry-IMU spatial and temporal measurement models follow established VIWO
formulations, allowing the analysis to focus on the previously unexplored
observability properties introduced by the 4WIS4WID kinematics and their
coupling with the VIO system. The main contributions of this work
are as follows:

\begin{itemize}

    \item We develop a preintegration based framework for online calibration of the 13 kinematic parameters of a 4WIS4WID robot and show that incorporating lateral no-slip constraints enables observability of the steering offsets, which remain unobservable in a drive-only model.

    \item We provide an observability analysis of the 4WIS4WID kinematic model, identifying three motion independent gauge freedoms associated with the $SE(2)$ placement of the odometry frame and characterizing motion dependent degeneracies that affect calibration of the remaining wheel intrinsics.

    \item We validate the proposed calibration and estimation framework in simulation and real world experiments, demonstrating improved pose estimation accuracy compared with using fixed kinematic parameters.

\end{itemize}
\section{4WIS4WID VISUAL-INERTIAL-WHEEL ODOMETRY}
\label{sec:viwo}
We employ the standard Multi-State Constraint Kalman Filter (MSCKF)
formulation~\cite{Mourikis2007}, consisting of the current inertial state
$\mathbf{x}_{I_k}$ and a sliding window of cloned IMU poses
$\mathbf{x}_{C_k}$. Camera-IMU spatiotemporal calibration is assumed known.
To enable online calibration of the wheel odometry, the MSCKF state is
augmented with the odometry-IMU extrinsics $\mathbf{x}_{WE}$, temporal offset
${}^Ot_I$, and the 4WIS4WID kinematic parameters $\mathbf{x}_{WI}$:
\begin{equation}
    \mathbf{x}_k =
    \begin{bmatrix}
        \mathbf{x}_{I_k}^\top &
        \mathbf{x}_{C_k}^\top &
        \mathbf{x}_{WE}^\top &
        {}^Ot_I &
        \mathbf{x}_{WI}^\top
    \end{bmatrix}^\top .
    \label{eq:augmented_state}
\end{equation}
The inertial and cloned states are
\begin{align}
    \mathbf{x}_{I_k} &=
    \begin{bmatrix}
        {}_{G}^{I_k}\bar q^\top &
        {}^{G}\mathbf{p}_{I_k}^\top &
        {}^{G}\mathbf{v}_{I_k}^\top &
        \mathbf{b}_g^\top &
        \mathbf{b}_a^\top
    \end{bmatrix}^\top, \\
    \mathbf{x}_{C_k} &=
    \begin{bmatrix}
        {}_{G}^{I_{k-1}}\bar q^\top &
        {}^{G}\mathbf{p}_{I_{k-1}}^\top &
        \cdots &
        {}_{G}^{I_{k-n}}\bar q^\top &
        {}^{G}\mathbf{p}_{I_{k-n}}^\top
    \end{bmatrix}^\top ,
\end{align}
where ${}_{G}^{I_k}\bar q$ represents the orientation of the IMU frame
$\{I\}$ with respect to the global frame $\{G\}$,
${}^{G}\mathbf{p}_{I_k}$ and ${}^{G}\mathbf{v}_{I_k}$ are the IMU position
and velocity, and $\mathbf{b}_g$ and $\mathbf{b}_a$ are the gyroscope and
accelerometer biases. The nominal and error states are related through
$\mathbf{x}=\hat{\mathbf{x}}\boxplus\tilde{\mathbf{x}}$.

IMU propagation, pose cloning, and visual feature updates follow the standard
MSCKF formulation~\cite{Mourikis2007}. The remainder of this section derives
the wheel odometry measurement model and the Jacobians required for online
calibration of the 4WIS4WID-specific parameters.

\subsection{4WIS4WID Kinematic Model}
\label{sec:kinematic_model}

Each wheel is equipped with steering and drive encoders providing
\begin{equation}
    \delta_{m_i} = \delta_i + \delta_{o_i} + n_{\delta_i},
    \qquad
    \omega_{md_i} = \omega_{d_i} + n_{d_i},
\end{equation}
where $\delta_i$ and $\omega_{d_i}$ are the true steering angle and drive
angular velocity of wheel $i$, $\delta_{o_i}$ is its constant steering offset,
and $n_{\delta_i}$ and $n_{d_i}$ are zero-mean measurement noises.
For $i=1,\ldots,4$, these measurements constrain the planar body twist
$\bm{\xi}=[v_x\ v_y\ \omega]^\top$ expressed in the odometry frame $\{O\}$.

Under the rigid body assumption, the velocity of the ground contact point of wheel $i$, expressed in the odometry frame $\{O\}$, is
\begin{equation}
    \begin{gathered}
    \mathbf{u}_i =
    \begin{bmatrix} v_x - y_{w_i}\omega \\ v_y + x_{w_i}\omega \end{bmatrix}
    = \begin{bmatrix} \mathbf{I}_2 & J\mathbf{p}_i \end{bmatrix} \bm\xi,
    \\
    J = \begin{bmatrix} 0 & -1 \\ 1 & 0 \end{bmatrix},
    \quad
    \mathbf{p}_i = \begin{bmatrix} x_{w_i} \\ y_{w_i} \end{bmatrix}.
    \end{gathered}
    \label{eq:contact_velocity}
\end{equation}
For a conventional wheel rolling without slipping, this velocity is constrained
to lie along the rolling direction
$\mathbf{d}_i = [\cos\delta_i\ \ \sin\delta_i]^\top$ with magnitude
$v_i = \omega_{d_i} r_i$, giving the standard 4WIS4WID kinematic
constraint $\mathbf{u}_i = \mathbf{d}_i\, v_i$ ~\cite{Lee2015JEng}.
Note that $\mathbf{u}_i$ comprises \emph{two} scalar equations
per wheel: the four wheels therefore impose eight constraints on a
2 DoF planar twist.

Premultiplying $\mathbf{u}_i$ by $R(\delta_i)^\top$ resolves the
constraint into the wheel's own frame and separates it into a longitudinal and
a lateral component,
\begin{equation}
    \mathbf{a}_i^\top \bm\xi = v_i,
    \qquad
    \mathbf{b}_i^\top \bm\xi = 0
    \label{eq:wheel_constraints}
\end{equation}
\begin{equation}
    \mathbf{a}_i^\top =
    \begin{bmatrix} \cos\delta_i & \sin\delta_i & c_i \end{bmatrix},
    \qquad
    \mathbf{b}_i^\top =
    \begin{bmatrix} -\sin\delta_i & \cos\delta_i & \bar c_i \end{bmatrix}
    \label{eq:ab_rows}
\end{equation}
\begin{equation}
    c_i = -y_{w_i}\cos\delta_i + x_{w_i}\sin\delta_i,
    \qquad
    \bar c_i = y_{w_i}\sin\delta_i + x_{w_i}\cos\delta_i
    \label{eq:c_definitions}
\end{equation}
The first of \eqref{eq:wheel_constraints} states that the contact-point velocity
projected onto the rolling direction equals the measured wheel speed; the second
states that its lateral component vanishes, i.e. the wheel does not slide
sideways.

Stacking \eqref{eq:wheel_constraints} over the four wheels yields
\begin{equation}
    \underbrace{\begin{bmatrix}
        \omega_{d_1} r_1 \\ \omega_{d_2} r_2 \\
        \omega_{d_3} r_3 \\ \omega_{d_4} r_4 \\ \mathbf{0}_{4\times1}
    \end{bmatrix}}_{\bar{\mathbf{v}}}
    =
    \underbrace{[\,\mathbf{a}_1\ \mathbf{a}_2\ \mathbf{a}_3\ \mathbf{a}_4\
        \mathbf{b}_1\ \mathbf{b}_2\ \mathbf{b}_3\ \mathbf{b}_4\,]^\top
        }_{\bar A(\bm\delta,\, \mathbf{x}_w,\, \mathbf{y}_w)
           \,=\, [A;\, B]}
    \underbrace{\begin{bmatrix} v_x \\ v_y \\ \omega \end{bmatrix}}_{\bm\xi}
    \label{eq:forward_map}
\end{equation}
and the body twist follows in a least-squares sense as
\begin{equation}
    \bm\xi = \bar A^\dagger \bar{\mathbf{v}}
           = (\bar A^\top \bar A)^{-1} \bar A^\top \bar{\mathbf{v}},
    \qquad
    \mathbf{e}_\perp = \bar{\mathbf{v}} - \bar A \bm\xi
    \label{eq:twist_ls}
\end{equation}
where $x_{w_i}$ and $y_{w_i}$ denote the position of wheel $i$ relative to the
baselink frame $\{O\}$. The residual $\mathbf{e}_\perp$ spans a
five-dimensional subspace and quantifies the extent to which the eight encoder
readings fail to be explained by any single rigid body twist, it is identically
zero for a kinematically exact, slip-free configuration. The pseudo-inverse in \eqref{eq:twist_ls} requires $\bar A\in\mathbb{R}^{8\times3}$ to have full column rank, $\operatorname{rank}(\bar A)=3$, so that $\bar A^\top\bar A$ is invertible.

\begin{remark}
\label{rem:lee2015}
We adopt the wheel frame form because it separates the
longitudinal and lateral constraints, whose distinct roles in parameter
observability are established in Section~\ref{sec:observability}. The first two
preclude the closed-form pseudo-inverse available under a symmetric
layout~\cite{Lee2015JEng}, so \eqref{eq:twist_ls} is evaluated numerically.
\end{remark}

\begin{remark}
\label{rem:naive_reduction}
Discarding the lateral rows of \eqref{eq:forward_map} yields the reduced map
$\mathbf{v} = A\bm\xi$, in which each wheel contributes only its rolling
measurement. This reduction is exact whenever the no-slip assumption holds, and
is the form implied by treating the four wheel speeds as the  odometric
measurements. Section~\ref{sec:lateral_necessity} shows that the steering offsets become
unobservable under every motion, and the reduced map loses rank under lateral
motion.
\end{remark}

\subsection{Odometry Preintegration}
Wheel odometry typically runs at 200-500Hz, performing EKF update for every measurment would be computationally expensive, wheel odometry measurments are therefore preintegrated between two consectuive camera frames at times $t_k$ and $t_{k+1}$ to obtain relative motion increasment. The increasment is expressed in the starting integration frame $\{O_k\}$ and re-initialized to zero at each new frame while 2D orientation evolves continously over integration interval and this evolving yaw is what rotates the body frame velocity into the position increment. Robot velocities are then integrated from $t_k$ to $t_{k+1}$ to obtain relative 2D pose measurment:
\begin{equation}
    \mathbf{z}_{k+1} =
\begin{bmatrix}
{}^{O_{k+1}}_{O_k}\theta \\[2pt]
{}^{O_k}\mathbf{p}_{O_{k+1}}
\end{bmatrix}
=
\begin{bmatrix}
\int_{t_k}^{t_{k+1}} {}^{O_t}\omega \, dt \\[4pt]
\int_{t_k}^{t_{k+1}} R\!\left({}^{O_t}_{O_k}\theta\right) {}^{O_t}\mathbf{v} \, dt
\end{bmatrix}
\label{eq:kinematic_model}
\end{equation}
\begin{equation}
    =: \mathbf{g}\big(\bm{\delta}_{i(k:k+1)},\ \bm{\omega}_{d_i(k:k+1)},\ \mathbf{x}_{WI}\big) \label{eq:odometry_measurements}
\end{equation}
where ${}^{O_t}\mathbf{v} = \begin{bmatrix} {}^{O_t}v_x & {}^{O_t}v_y \end{bmatrix}^T$, $R(\cdot)$ is the standard 2D rotation matrix, ${}^{O_\tau}_{O_k}\theta$ is local yaw angle, ${}^{O_k}{x}_{O_\tau}$ and ${}^{O_k}{y}_{O_\tau}$ are position of $\{O_{\tau}\}$ in starting frame  $\{O_{k}\}$. 

If all spatiotemporal parameters between the wheel odometry, camera, and IMU are known, the odometry measurements introduced above can be directly incorporated into the MSCKF framework. However, in practice, mobile robot kinematic parameters can vary over time due to inaccurate robot calibration, mechanical inaccuracies, and vibrations. To account for such time-varying parameters, this work performs online calibration of the 4WIS4WID robot's intrinsic $\mathbf{x}_{WI}$ and extrinsic $\mathbf{x}_{WE}$ parameters, including the kinematic parameters of the wheel odometry model and the spatial transformation between the odometry and IMU frames. The wheel odometry–IMU temporal offset model is adopted from~\cite{Lee2020IROS} and is not derived here, \(^{O}t_I\) remains part of the filter state and is estimated online. The camera–IMU spatiotemporal calibration is assumed known.
\subsection{Odometry Measurement with Intrinsics}
\label{sec:intrinsic_jacobians}

The integration \eqref{eq:odometry_measurements} depends on the wheel
intrinsics, and should in principle be repeated whenever a new estimate becomes
available which would reduce the efficiency of preintegration. We instead
linearise the preintegrated measurement about the current estimate
$\hat{\mathbf{x}}_{WI}$, accounting for both the linearisation error and the
encoder noise,
\begin{equation}
    \resizebox{\columnwidth}{!}{$\displaystyle
    \mathbf{z}_{k+1} \simeq
    \mathbf{g}\big(\bm\delta_{m(k:k+1)},\, \bm\omega_{md(k:k+1)},\,
                   \hat{\mathbf{x}}_{WI}\big)
    + \frac{\partial\mathbf{g}}{\partial\tilde{\mathbf{x}}_{WI}}
      \tilde{\mathbf{x}}_{WI}
    + \frac{\partial\mathbf{g}}{\partial\mathbf{n}_w}\mathbf{n}_w
    \label{eq:linearized_measurement}$}
\end{equation}
with the sixteen intrinsics collected as $\mathbf{x}_{WI} :=
    \begin{bmatrix} \mathbf{r}^\top & \mathbf{x}_w^\top & \mathbf{y}_w^\top
                    & \bm\delta_o^\top \end{bmatrix}^\top \in \mathbb{R}^{16}$,
$\mathbf{n}_w$ collects drive, steering and lateral slip noise $\mathbf{n}_\tau$ of all subintervals in $[t_k, t_{k+1}]$.
Three of these are unobservable for any motion and are fixed by convention which is established in Section~\ref{sec:observability}.

For full-column-rank $\bar A$, the sensitivity of the least-squares twist
$\bm\xi = \bar A^\dagger\bar{\mathbf{v}}$ to a parameter $p$ entering through
$\bar A$ is
\begin{equation}
    \frac{\partial\bm\xi}{\partial p}
    = -\bar A^\dagger D\bm\xi
      + (\bar A^\top\bar A)^{-1} D^\top \mathbf{e}_\perp,
    \qquad D := \frac{\partial\bar A}{\partial p}
    \label{eq:master_identity}
\end{equation}
the second term being present only because $\bar A$ has more rows than columns;
it vanishes identically for a square kinematic map.

Only rows $j$ and $4{+}j$ of $\bar A$ depend on the parameters of wheel $j$. It
is therefore convenient to collect the corresponding columns of the
pseudo-inverse and components of the residual,
\begin{equation}
    \bar A^\dagger_j := \begin{bmatrix} [\bar A^\dagger]_{:,j}
        & [\bar A^\dagger]_{:,4+j}\end{bmatrix} \in \mathbb{R}^{3\times2},
    \qquad
    \mathbf{e}_{\perp,j} := \begin{bmatrix} e_{\perp,j} \\
        e_{\perp,4+j}\end{bmatrix}
    \label{eq:per_wheel_blocks}
\end{equation}
Differentiating \eqref{eq:ab_rows} with respect to $\delta_j$ shows that the two
rows rotate into one another,
\begin{equation}
    \frac{\partial}{\partial\delta_j}
    \begin{bmatrix}\mathbf{a}_j^\top \\ \mathbf{b}_j^\top\end{bmatrix}
    = -J
    \begin{bmatrix}\mathbf{a}_j^\top \\ \mathbf{b}_j^\top\end{bmatrix},
    \qquad
    \frac{\partial c_j}{\partial\mathbf{p}_j^\top} = -\mathbf{n}_j^\top,
    \qquad
    \frac{\partial\bar c_j}{\partial\mathbf{p}_j^\top} = \mathbf{d}_j^\top
    \label{eq:row_derivatives}
\end{equation}
Applying \eqref{eq:master_identity} blockwise then gives the instantaneous
twist Jacobian $H_{\bm\xi,\mathbf{x}_{WI}} :=
\partial\bm\xi/\partial\mathbf{x}_{WI} \in \mathbb{R}^{3\times16}$. All
quantities below are evaluated at $\hat{\mathbf{x}}_{WI}$ and at the measured
$\delta_{m_j}$, $\omega_{md_j}$.

The radii enter $\bar{\mathbf{v}}$ only, leaving $\bar A$ unchanged:
\begin{equation}
    \frac{\partial\bm\xi}{\partial r_j} = \omega_{md_j}\,[\bar A^\dagger]_{:,j}
    \label{eq:jac_radii}
\end{equation}
The steering offsets enter $\bar A$ only, through
$\partial\delta_j/\partial\delta_{o_j} = -1$:
\begin{equation}
    \frac{\partial\bm\xi}{\partial\delta_{o_j}}
    = \bar A^\dagger_j
      \begin{bmatrix} \mathbf{b}_j^\top\bm\xi \\
                      -\mathbf{a}_j^\top\bm\xi \end{bmatrix}
    - (\bar A^\top\bar A)^{-1}
      \begin{bmatrix}\mathbf{b}_j & -\mathbf{a}_j\end{bmatrix}
      \mathbf{e}_{\perp,j}
    \label{eq:jac_steering}
\end{equation}
The wheel positions enter $\bar A$ through $c_j$ and $\bar c_j$:
\begin{equation}
    \frac{\partial\bm\xi}{\partial\mathbf{p}_j^\top}
    = \Big(\omega\,\bar A^\dagger_j
           - (\bar A^\top\bar A)^{-1}\mathbf{e}_3\,\mathbf{e}_{\perp,j}^\top
      \Big)
      \begin{bmatrix} \mathbf{n}_j^\top \\ -\mathbf{d}_j^\top \end{bmatrix}
    \label{eq:jac_position}
\end{equation}
whose two columns are placed in the $x_{w_j}$ and $y_{w_j}$ positions of
$H_{\bm\xi,\mathbf{x}_{WI}}$. $\mathbf{e}_3 = [0\ 0\ 1]^\top$ selects the yaw rate column of $\bar A$, the only column through which the wheel position enters; the two columns of \eqref{eq:jac_position} are placed in the $x_{w_j}$ and $y_{w_j}$ positions of $H_{\bm\xi,\mathbf{x}_{WI}}$.

\begin{remark}
\label{rem:position_rank}
The $2\times2$ factor in \eqref{eq:jac_position} is orthogonal, so both
components of $\mathbf{p}_j$ influence the twist independently at every
instant. Under the reduced map $\mathbf{v} = A\bm\xi$ this is not the case:
$x_{w_j}$ and $y_{w_j}$ then enter only through the single scalar $c_j$, and
only that combination is observable.
\end{remark}

Collect the drive and steering encoder noises together with a lateral term,
$\mathbf{n} = [\,\mathbf{n}_d^\top\ \ \mathbf{n}_\delta^\top\ \
\mathbf{n}_{\ell}^\top\,]^\top \in \mathbb{R}^{12}$, where
$\mathbf{n}_{\ell}$ enters the lower block of $\bar{\mathbf{v}}$ and represents
residual lateral slip not captured by the rolling constraint
$\mathbf{u}_i$. Since $\omega_{d_j} = \omega_{md_j} - n_{d_j}$
and $\partial\delta_j/\partial n_{\delta_j} =
\partial\delta_j/\partial\delta_{o_j} = -1$,
\begin{equation}
    \frac{\partial\bm\xi}{\partial n_{d_j}} = -\hat r_j\,[\bar A^\dagger]_{:,j},
    \quad
    \frac{\partial\bm\xi}{\partial n_{\delta_j}}
      = \frac{\partial\bm\xi}{\partial\delta_{o_j}},
    \quad
    \frac{\partial\bm\xi}{\partial n_{\ell_j}} = [\bar A^\dagger]_{:,4+j}
    \label{eq:jac_noise}
\end{equation}
so the steering block is obtained without additional computation. Stacking
these gives $H_{\bm\xi,\mathbf{n}} \in \mathbb{R}^{3\times12}$, with
$Q_\tau = \operatorname{diag}(\sigma_d^2\mathbf{I}_4,\
\sigma_\delta^2\mathbf{I}_4,\ \sigma_\ell^2\mathbf{I}_4)$.

Let $\tau$ index the encoder samples in $[t_k, t_{k+1}]$, with sub-interval
length $\Delta t = t_{\tau+1}-t_\tau$, and let
$\mathbf{g}_\tau = [\theta_\tau\ x_\tau\ y_\tau]^\top$ denote the increment
preintegrated up to $t_\tau$, so that $\mathbf{g}_k=\mathbf{0}$ and
$\mathbf{g}_{k+1}=\mathbf{g}$.
The intrinsics affect every sub-interval of the preintegration, so
$\partial\mathbf{g}/\partial\tilde{\mathbf{x}}_{WI}$ must be accumulated
alongside the increment itself. With $\theta_m := \theta_\tau - \omega\Delta
t/2$ and $\rho := (2/\omega) \sin(\omega \Delta t / 2)$ the
position increment over the sub-interval is
\begin{equation}
\begin{split}
    \Delta x &= \rho\,(v_x\cos\theta_m + v_y\sin\theta_m), \\[4pt]
    \Delta y &= \rho\,(-v_x\sin\theta_m + v_y\cos\theta_m)
\end{split}
\label{eq:preint_closed_form}
\end{equation}
This is the singularity-free form of the position increment: the body velocity
rotated by the mid-interval yaw $\theta_m$ and scaled by the arc factor $\rho$.
It is obtained from the exact integral of \eqref{eq:kinematic_model} over
$[t_\tau, t_{\tau+1}]$ by applying sum-to-product identities, which removes the
$1/\omega$ singularity present in the direct evaluation and is therefore valid
for the zero-yaw-rate motions analysed in Section~\ref{sec:degenerate}.

The increment Jacobians with respect to the running pose and the twist are
\begin{equation}
\begin{split}
    \Phi_{tr,\tau} &=
    \begin{bmatrix} 1 & 0 & 0 \\ \Delta y & 1 & 0 \\
                    -\Delta x & 0 & 1 \end{bmatrix}, \\[4pt]
    \Phi_{\bm\xi,\tau} &=
    \begin{bmatrix}
        0 & 0 & -\Delta t \\
        \rho\cos\theta_m & \rho\sin\theta_m &
            \frac{\rho'}{\rho}\Delta x - \frac{\Delta t}{2}\Delta y \\
        -\rho\sin\theta_m & \rho\cos\theta_m &
            \frac{\rho'}{\rho}\Delta y + \frac{\Delta t}{2}\Delta x
    \end{bmatrix}
\end{split}
\label{eq:Phi_tr_xi}
\end{equation}
with $\rho' = \tfrac{\Delta t^2}{2}(z\cos z - \sin z)/z^2$ and $z :=
\omega\Delta t/2$. Composing with the instantaneous Jacobians,
\begin{equation}
    \Phi_{WI,\tau} = \Phi_{\bm\xi,\tau} H_{\bm\xi,\mathbf{x}_{WI},\tau},
    \qquad
    \Phi_{n,\tau} = \Phi_{\bm\xi,\tau} H_{\bm\xi,\mathbf{n},\tau}
    \label{eq:Phi_WI_n}
\end{equation}
and propagating from zero initial conditions at $\tau = k$,
\begin{align}
    \frac{\partial\mathbf{g}_{\tau+1}}{\partial\tilde{\mathbf{x}}_{WI}}
    &= \Phi_{tr,\tau}
       \frac{\partial\mathbf{g}_\tau}{\partial\tilde{\mathbf{x}}_{WI}}
       + \Phi_{WI,\tau}
    \label{eq:dg_dxWI_recursion} \\
    P_{m,\tau+1} &= \Phi_{tr,\tau} P_{m,\tau} \Phi_{tr,\tau}^\top
       + \Phi_{n,\tau} Q_\tau \Phi_{n,\tau}^\top
    \label{eq:Pm_recursion}
\end{align}
yields at $\tau = k{+}1$ the quantities required in
\eqref{eq:linearized_measurement}. Factoring the kinematics-agnostic
integration Jacobian $\Phi_{\bm\xi,\tau}$ from the kinematics-specific twist
Jacobians is what keeps the sixteen-parameter case tractable; the inlined form
used for differential drive VIWO does not scale to this many
intrinsics.

\subsection{Odometry Measurement with Extrinsics}
\label{sec:extrinsics}

The preintegrated measurement \eqref{eq:odometry_measurements} is a 2D relative
pose in the odometry frame $\{O\}$, whereas the filter state contains 3D IMU poses. Relating the two requires the
odometry-IMU spatial extrinsics $\mathbf{x}_{WE}$ and the temporal offset
${}^Ot_I$.

Since this relation is independent of the wheel kinematic model, we adopt the
measurement function and corresponding Jacobians derived in
\cite{Lee2020IROS} without modification. We denote the predicted
odometry measurement as
\begin{equation}
    \mathbf{h}_{k+1}
    =
    h\!\left(
        \mathbf{x}_{I_{k+1}},
        \mathbf{x}_{C_{k+1}},
        \mathbf{x}_{WE},
        {}^Ot_I
    \right),
    \label{eq:h_extrinsics}
\end{equation}
whose first-order perturbation can be written as
\begin{equation}
    \delta\mathbf{h}_{k+1}
    \simeq
    \mathbf{H}_{I}\tilde{\mathbf{x}}_{I_{k+1}}
    +
    \mathbf{H}_{C}\tilde{\mathbf{x}}_{C_{k+1}}
    +
    \mathbf{H}_{WE}\tilde{\mathbf{x}}_{WE}
    +
    \mathbf{H}_{t}\,{}^O\tilde t_I ,
    \label{eq:h_extrinsics_linearized}
\end{equation}
where $\mathbf{H}_{I}$, $\mathbf{H}_{C}$, $\mathbf{H}_{WE}$, and
$\mathbf{H}_{t}$ are the corresponding Jacobian blocks given in
\cite{Lee2020IROS}. Although unchanged from the differential drive
formulation, these blocks are retained in the complete measurement Jacobian
since the odometry-IMU extrinsics and temporal offset enter the observability
analysis of Section~\ref{sec:observability}.

\begin{remark}
\label{rem:gauge}
The extrinsics $\mathbf{x}_{WE}$ and the wheel geometry
$\{\mathbf{x}_w,\mathbf{y}_w,\bm\delta_o\}$ are not independently observable.
The placement of $\{O\}$ on the robot body is an arbitrary choice of
coordinates. Therefore, any planar rigid transformation of the wheel array can
be absorbed by a compensating change in ${}^O_I\bar q$ and
${}^O\mathbf{p}_I$ without altering the measurements. This introduces three
motion-independent unobservable directions corresponding to the $SE(2)$
freedom in the placement of $\{O\}$, which are characterised in
Section~\ref{sec:observability}.
\end{remark}

\subsection{Odometry Measurement Update}
\label{sec:odom_update}

Sections~\ref{sec:intrinsic_jacobians} and \ref{sec:extrinsics} supply the two halves
of the update: the preintegrated measurement $\mathbf{z}_{k+1}$ with its
intrinsic Jacobian $\partial\mathbf{g}/\partial\tilde{\mathbf{x}}_{WI}$ and
accumulated covariance $\mathbf{P}_m$, and the prediction $\mathbf{h}(\cdot)$
relating that 2D relative pose to the 3D state. Equating them,
\begin{equation}
    \resizebox{\columnwidth}{!}{$\displaystyle
    \mathbf{g}\big(\bm\delta_{m(k:k+1)},\, \bm\omega_{md(k:k+1)},\,
                   \mathbf{x}_{WI}\big)
    = \mathbf{h}\big(\mathbf{x}_{I_{k+1}}, \mathbf{x}_{C_{k+1}},
                     \mathbf{x}_{WE}, {}^Ot_I\big)
    \label{eq:update_constraint}
    $}
\end{equation}
and linearising each side about the current estimates, the left via
\eqref{eq:linearized_measurement}, the right via
\eqref{eq:h_extrinsics_linearized} which yields the residual
\begin{equation}
    \resizebox{\columnwidth}{!}{$\displaystyle
    \tilde{\mathbf{z}}_{k+1} :=
    \mathbf{g}\big(\bm\delta_{m(k:k+1)},\, \bm\omega_{md(k:k+1)},\,
                   \hat{\mathbf{x}}_{WI}\big)
    - \mathbf{h}\big(\hat{\mathbf{x}}_{I_{k+1}}, \hat{\mathbf{x}}_{C_{k+1}},
                     \hat{\mathbf{x}}_{WE}, {}^O\hat t_I\big)
    \label{eq:odom_residual}
    $}
\end{equation}
and the linearised measurement model
\begin{equation}
    \resizebox{\columnwidth}{!}{$\displaystyle
    \tilde{\mathbf{z}}_{k+1} \simeq
    \underbrace{
    \begin{bmatrix}
        \dfrac{\partial\mathbf{h}}{\partial\tilde{\mathbf{x}}_{I}} &
        \dfrac{\partial\mathbf{h}}{\partial\tilde{\mathbf{x}}_{C}} &
        \dfrac{\partial\mathbf{h}}{\partial\tilde{\mathbf{x}}_{WE}} &
        \dfrac{\partial\mathbf{h}}{\partial{}^O\tilde t_I} &
        -\dfrac{\partial\mathbf{g}}{\partial\tilde{\mathbf{x}}_{WI}}
    \end{bmatrix}}_{\mathbf{H}_{k+1}}
    \tilde{\mathbf{x}}_{k+1}
    - \frac{\partial\mathbf{g}}{\partial\mathbf{n}_w}\mathbf{n}_w
    \label{eq:odom_linearized}
    $}
\end{equation}
with block ordering following \eqref{eq:augmented_state}. The intrinsic block
carries a negative sign because the intrinsics enter through the measurement
rather than its prediction, and has dimension $3\times16$ against $3\times3$
for a differential drive platform; all remaining blocks are structurally
identical to that case.

The effective measurement noise is the accumulated covariance $\mathbf{P}_m$
from \eqref{eq:Pm_recursion}, which already propagates the driving, steering
and lateral-slip noise through the least-squares twist recovery.
Following~\cite{Lee2020IROS}, candidate measurements are screened by a
Mahalanobis distance test with three degrees of freedom,
\begin{equation}
    \tilde{\mathbf{z}}_{k+1}^\top
    \big(\mathbf{H}_{k+1}\mathbf{P}_{k+1|k}\mathbf{H}_{k+1}^\top
         + \mathbf{P}_m\big)^{-1}
    \tilde{\mathbf{z}}_{k+1} < \chi^2_{3,\,\alpha}
    \label{eq:chi2}
\end{equation}
so that intervals corrupted by unmodelled effects such as wheel slippage are
rejected.

\begin{remark}
\label{rem:eperp_gating}
Unlike differential drive, the 4WIS4WID system provides a consistency
signal through $\mathbf{e}_\perp$. While zero for a square kinematic map,
$\mathbf{e}_\perp$ spans a five dimensional subspace here and detects encoder
measurements inconsistent with a rigid body twist. Thus, timesteps can be
gated on $\|\mathbf{e}_\perp\|$ during preintegration, rejecting slip before it
corrupts the accumulated increment.
\end{remark}

\section{Observability Analysis}
\label{sec:observability}

Observability analysis determines which states and parameters are recoverable
from a given measurement and is a prerequisite for consistent
estimator design~\cite{Hesch2014, Huang2012}. Following~\cite{Lee2020IROS} we
consider a state containing a single cloned pose and a single point feature,
\begin{equation}
    \resizebox{\columnwidth}{!}{$\displaystyle
    \mathbf{x}_k =
    \begin{bmatrix} \mathbf{x}_{I_k}^\top & \mathbf{x}_{C_k}^\top
        & \mathbf{x}_{etc}^\top \end{bmatrix}^\top,
    \;\;
    \mathbf{x}_{etc} =
    \begin{bmatrix} \mathbf{x}_{WE}^\top & {}^Ot_I & \mathbf{x}_{WI}^\top
        & {}^G\mathbf{p}_f^\top \end{bmatrix}^\top
    \label{eq:obs_state}
    $}
\end{equation}
for which the observability matrix over $[t_0, t_k]$ is
\begin{equation}
    \mathbf{M} = \begin{bmatrix}
        \mathbf{H}_0^\top &
        \big(\mathbf{H}_1\bm\Phi_{(t_1,t_0)}\big)^\top & \cdots &
        \big(\mathbf{H}_k\bm\Phi_{(t_k,t_0)}\big)^\top
    \end{bmatrix}^\top
    \label{eq:observability_matrix}
\end{equation}
with $\mathbf{H}_k$ the stacked visual and wheel measurement Jacobians
\eqref{eq:odom_linearized}. Any $\mathbf{n}$ with $\mathbf{M}\mathbf{n} =
\mathbf{0}$ is an unobservable direction: perturbing the initial state along
$\mathbf{n}$ leaves every measurement unchanged, so no estimator can resolve it
and the corresponding covariance is unbounded.

The state transition matrix $\bm\Phi_{(t_k,t_0)}$ is not immediate when clones
are included in the state, and we adopt the construction of~\cite{Lee2020IROS},
in which cloning and propagation are unified into a single linear map and the
coincidence of the initial IMU pose and initial clone is imposed as an
error-state constraint:
\begin{equation}
\resizebox{\columnwidth}{!}{$\displaystyle
    \tilde{\mathbf{x}}_{k+1} =
    \begin{bmatrix}
        \bm\Phi_{I_{11}} & \mathbf{0}_3 & \mathbf{0}_3
            & \bm\Phi_{I_{14}} & \mathbf{0}_3
            & \mathbf{0}_3 & \mathbf{0}_3 & \mathbf{0}_{3\times26} \\
        \bm\Phi_{I_{21}} & \mathbf{I}_3
            & \bm\Phi_{I_{23}} & \bm\Phi_{I_{24}} & \bm\Phi_{I_{25}}
            & \mathbf{0}_3 & \mathbf{0}_3 & \mathbf{0}_{3\times26} \\
        \bm\Phi_{I_{31}} & \mathbf{0}_3 & \mathbf{I}_3
            & \bm\Phi_{I_{34}} & \bm\Phi_{I_{35}}
            & \mathbf{0}_3 & \mathbf{0}_3 & \mathbf{0}_{3\times26} \\
        \mathbf{0}_3 & \mathbf{0}_3 & \mathbf{0}_3 & \mathbf{I}_3
            & \mathbf{0}_3 & \mathbf{0}_3 & \mathbf{0}_3
            & \mathbf{0}_{3\times26} \\
        \mathbf{0}_3 & \mathbf{0}_3 & \mathbf{0}_3 & \mathbf{0}_3
            & \mathbf{I}_3 & \mathbf{0}_3 & \mathbf{0}_3
            & \mathbf{0}_{3\times26} \\
        \mathbf{0}_3 & \mathbf{0}_3 & \mathbf{0}_3
            & \bm\Psi_{C_{14}} & \mathbf{0}_3
            & \bm\Psi_{C_{11}} & \mathbf{0}_3 & \mathbf{0}_{3\times26} \\
        \mathbf{0}_3 & \mathbf{0}_3 & \bm\Psi_{C_{23}}
            & \bm\Psi_{C_{24}} & \bm\Psi_{C_{25}}
            & \bm\Psi_{C_{21}} & \mathbf{I}_3 & \mathbf{0}_{3\times26} \\
        \mathbf{0}_{26\times3} & \mathbf{0}_{26\times3} & \mathbf{0}_{26\times3}
            & \mathbf{0}_{26\times3} & \mathbf{0}_{26\times3}
            & \mathbf{0}_{26\times3} & \mathbf{0}_{26\times3} & \mathbf{I}_{26}
    \end{bmatrix}
    \tilde{\mathbf{x}}_0
$}
    \label{eq:state_transition}
\end{equation}
where all blocks are evaluated over $(t_{k+1},t_0)$, the block rows correspond
in order to the IMU orientation, position, velocity, gyroscope and
accelerometer biases, the cloned orientation and position, and finally
$\mathbf{x}_{etc}$. The $\bm\Phi_{I_{\ast}}$ are the standard inertial
error-state transition blocks and $\bm\Psi_{C_{\ast}}(t_{k+1},t_0) =
\bm\Phi_{I_{\ast}}(t_k,t_0)$ encode the clone constraint, both given explicitly
in~\cite{Lee2020IROS}. Two observations justify reusing
\eqref{eq:state_transition} without rederivation: $\bm\Phi$ describes only
inertial propagation and the clone/marginalise bookkeeping of the sliding
window, so no wheel measurement enters it and it is independent of the
kinematic model; and all calibration parameters are constant, so their
transition block is the identity regardless of their number. Only its dimension
changes, from $\mathbf{I}_{13}$ for the differential drive case to
$\mathbf{I}_{26}$ here.

\subsection{Structure of the Wheel Measurement Jacobian}
\label{sec:jacobian_structure}

Observability is evaluated at the true state, where the calibration is exact
and no slippage occurs. Hence $\mathbf{e}_\perp = \mathbf{0}$,
$\mathbf{a}_j^\top\bm\xi = v_j$ and $\mathbf{b}_j^\top\bm\xi = 0$, so the
second term of \eqref{eq:master_identity} vanishes identically and
\eqref{eq:jac_radii}-\eqref{eq:jac_position} reduce as follows.

\begin{proposition}[Per-wheel constraint structure]
\label{prop:per_wheel}
At the true state, the four intrinsics of wheel $j$ influence the recovered
twist only through the two columns collected in $\bar A^\dagger_j$ of
\eqref{eq:per_wheel_blocks}:
\begin{equation}
    \resizebox{\columnwidth}{!}{$\displaystyle
    \frac{\partial\bm\xi}{\partial\big[\tilde r_j\;\; \tilde\delta_{o_j}\;\;
        \tilde x_{w_j}\;\; \tilde y_{w_j}\big]}
    = \bar A^\dagger_j\, \mathbf{C}_j,
    \quad
    \mathbf{C}_j = \begin{bmatrix}
        \omega_{md_j} & 0 & \omega\,{\mathbf{d}_j^\perp}^\top \\[2pt]
        0 & -v_j & -\omega\,\mathbf{d}_j^\top
    \end{bmatrix}
    \label{eq:Cj}$}
\end{equation}
\end{proposition}

Since $\mathrm{rank}(\mathbf{C}_j) \le 2$, at least two distinct steering
configurations are required before all four intrinsics of a wheel can be
identified. The entries of \eqref{eq:Cj} also indicate the excitation each
parameter demands: the radius sensitivity scales with the wheel rate
$\omega_{md_j}$, the steering offset sensitivity with the wheel speed $v_j$,
and both position sensitivities with the yaw rate $\omega$. Every degeneracy
below is a mechanism by which $\bigcup_t \mathrm{rowspace}(\mathbf{C}_j(t))$
fails to span $\mathbb{R}^4$.

\subsection{Necessity of the Lateral Constraints}
\label{sec:lateral_necessity}

Remark~\ref{rem:naive_reduction} noted that discarding the lateral rows of
\eqref{eq:forward_map} yields a reduced map $\mathbf{v} = A\bm\xi$, exact under
the no-slip assumption. It is nonetheless inadmissible for calibration.

\begin{lemma}
\label{lem:drive_only}
Under $\bm\xi = A^\dagger\mathbf{v}$, the steering offsets $\bm\delta_o$ are
unobservable under every motion.
\end{lemma}

\begin{proof}
For the reduced map, the counterpart of \eqref{eq:jac_steering} is
$\partial\bm\xi/\partial\delta_{o_j} = (\mathbf{b}_j^\top\bm\xi)
[A^\dagger]_{:,j} - e_{\perp,j}(A^\top A)^{-1}\mathbf{b}_j$. With
$\mathbf{u}_j$ the contact point velocity \eqref{eq:contact_velocity},
\begin{equation}
    \mathbf{b}_j^\top\bm\xi = {\mathbf{d}_j^\perp}^\top\mathbf{u}_j
    \label{eq:lateral_slip}
\end{equation}
which is the lateral slip velocity of wheel $j$ and vanishes identically under
rolling without slipping. At the true state $\mathbf{e}_\perp = \mathbf{0}$, so
the second term vanishes also, and the four columns of $\mathbf{M}$ associated
with $\bm\delta_o$ are identically zero.
\end{proof}

The rolling measurements alone report only \emph{how fast} each wheel turns;
they never test whether the four steering angles are mutually consistent with a
single ICR, which is where steering offset
information resides. Retaining the lateral rows restores it: from
\eqref{eq:Cj} the offset enters through $[\bar A^\dagger]_{:,4+j}$ with
coefficient $-v_j$, non-zero whenever the wheel is rolling.

\begin{remark}
\label{rem:rank_crab}
The augmented map is also better conditioned: for $A$, a configuration with all
steering angles equal renders the first two columns parallel and $(A^\top
A)^{-1}$ undefined, whereas for $\bar A$ they remain independent for any
steering configuration. The pseudo-inverse is thus well defined in precisely
the lateral translation mode a 4WIS4WID platform is built to execute.
\end{remark}

\begin{lemma}
\label{lem:vins}
Under general motion, four directions are unobservable, corresponding to global
position and rotation about gravity, as in standard VINS~\cite{Hesch2014,
Yang2023TRO}.
\end{lemma}

\begin{proof}
The null basis of~\cite{Lee2020IROS}, extended by zeros in the wheel-intrinsic
block: the wheel measurements \eqref{eq:odometry_measurements} are expressed
entirely in the local frame $\{O_k\}$ and are invariant to global translation
and to rotation about gravity.
\end{proof}

\begin{lemma}
\label{lem:gauge}
Three further directions are unobservable \emph{under every motion},
corresponding to the $SE(2)$ freedom in placing $\{O\}$ on the robot body.
\end{lemma}

\begin{proof}
Let $\{O'\}$ have origin $\mathbf{q} = [q_x\ \ q_y]^\top$ and axes rotated by
$\beta$, expressed in $\{O\}$, and let $\bm\nu := [v_x\ \ v_y]^\top$. Then
\begin{equation}
    \resizebox{\columnwidth}{!}{$\displaystyle
    \mathbf{p}_j' = R(\beta)^\top(\mathbf{p}_j - \mathbf{q}),
    \quad
    \delta_j' = \delta_j - \beta,
    \quad
    \bm\xi' = \begin{bmatrix}
        R(\beta)^\top(\bm\nu + \omega J\mathbf{q}) \\ \omega
    \end{bmatrix}
    \label{eq:gauge_transform}$}
\end{equation}
and, since the physical encoder reading $\delta_{m_j} = \delta_j +
\delta_{o_j}$ is unaltered by a change of coordinates, $\delta_{o_j}' =
\delta_{o_j} + \beta$. Using $\mathbf{d}_j' = R(\beta)^\top\mathbf{d}_j$ and
$JR(\beta)^\top = R(\beta)^\top J$,
\begin{equation}
    v_j' = \mathbf{d}_j'^\top\big(\bm\nu' + \omega J\mathbf{p}_j'\big)
    = \mathbf{d}_j^\top\big[\bm\nu + \omega J\mathbf{p}_j\big] = v_j
    \label{eq:gauge_invariance}
\end{equation}
and identically $\mathbf{b}_j'^\top\bm\xi' = \mathbf{b}_j^\top\bm\xi$, so every
wheel measurement is unchanged. The prediction $\mathbf{h}(\cdot)$ is likewise
unchanged provided $\mathbf{x}_{WE}$ absorbs the frame change, which it can
since both ${}^O_I\bar q$ and ${}^O\mathbf{p}_I$ are estimated. Linearising
\eqref{eq:gauge_transform} for small $\beta = \alpha$ and $\mathbf{q}$ gives
the intrinsic block of the null basis,
\begin{equation}
    \tilde{\mathbf{r}} = \mathbf{0}, \;\;
    \tilde\delta_{o_j} = \alpha, \;\;
    \tilde x_{w_j} = \alpha y_{w_j} - q_x, \;\;
    \tilde y_{w_j} = -\alpha x_{w_j} - q_y
    \label{eq:gauge_null_basis}
\end{equation}
spanned by $\alpha$, $q_x$ and $q_y$.
\end{proof}

Lemma~\ref{lem:gauge} differs structurally from the degeneracies reported for
differential drive~\cite{Lee2020IROS} and skid-steering~\cite{Zuo2024TASE}
platforms, all of which are motion dependent: it holds under fully general
excitation, and has no differential drive analogue, the corresponding intrinsic
there being a scalar baselink rather than an $SE(2)$ object. The wheel geometry
is therefore observable only up to the choice of odometry frame. We fix the
gauge by holding $x_{w_1}$, $y_{w_1}$ and $\delta_{o_1}$ at their nominal
values and estimating the remaining thirteen intrinsics online; residual error
in these three is absorbed by $\mathbf{x}_{WE}$ without introducing bias.
Lemmas 2 and 3 identify seven motion independent unobservable directions under general motion: four standard VINS directions and three associated with the arbitrary placement of the odometry frame.

\subsection{Degenerate Motions}
\label{sec:degenerate}

The following follow from rank conditions on $\mathbf{C}_j$ in \eqref{eq:Cj},
and are summarised in Table~\ref{tab:degenerate_motions}.

\begin{lemma}
\label{lem:zero_yaw}
If $\omega \equiv 0$, all eight wheel positions are unobservable.
\end{lemma}

\begin{proof}
$\omega = 0$ annihilates both position columns of $\mathbf{C}_j$ for every $j$.
\end{proof}
This covers straight line and pure lateral motion where the wheel positions
enter only through the moment arm acting against the yaw rate, which without
rotation produces no signal.

\begin{lemma}
\label{lem:const_steering}
If all $\delta_j$ are held constant and $\omega \not\equiv 0$, two of the four intrinsics of every wheel
are unobservable, regardless of trajectory duration.
\end{lemma}

\begin{proof}
A constant steering configuration fixes the ICR
$\mathbf{c}$ in the body frame, so $\lambda_j := v_j/\omega = \|\mathbf{p}_j -
\mathbf{c}\|$ and $\eta_j := \omega_{md_j}/\omega = \lambda_j/r_j$ are constant
and $\mathbf{C}_j(t) = \omega(t)\,\mathbf{C}_j^0$ with $\mathbf{C}_j^0$ fixed.
The accumulated row space is the fixed 2D
$\mathrm{rowspace}(\mathbf{C}_j^0)$, whose null directions, parameterised by
$(\tilde x_{w_j}, \tilde y_{w_j})$, are
\begin{equation}
    \resizebox{\columnwidth}{!}{$\displaystyle
    \tilde r_j = \frac{\sin\delta_j\,\tilde x_{w_j}
        - \cos\delta_j\,\tilde y_{w_j}}{\eta_j},
    \quad
    \tilde\delta_{o_j} = -\frac{\cos\delta_j\,\tilde x_{w_j}
        + \sin\delta_j\,\tilde y_{w_j}}{\lambda_j}
    \label{eq:const_steering_null}$}
\end{equation}
\end{proof}
This is the most restrictive result in practice, excluding straight line, lateral
and constant radius circular motion, trajectories most naturally chosen
for calibration. Increasing speed or duration does not help, since $\omega(t)$
scales $\mathbf{C}_j$ without altering its row space.

\begin{table}[t]
\centering
\caption{Degenerate motions and the corresponding unobservable calibration
parameters.}
\label{tab:degenerate_motions}
\begin{tabular}{ll}
\hline
\textbf{Motion} & \textbf{Unobservable} \\
\hline
General motion & 4 inertial $+$ 3 gauge \\
Zero yaw rate ($\omega \equiv 0$) & $\mathbf{x}_w$, $\mathbf{y}_w$ \\
Constant steering configuration & 2 per wheel (8 total) \\
Wheel $j$ not rolling & $r_j$, $\delta_{o_j}$ \\
Pure translation & ${}^O\mathbf{p}_I$ \\
One axis rotation & ${}^O\mathbf{p}_I$ along axis \\
Constant local $\mathbf{v}$, $\bm\omega$ & ${}^Ot_I$ \\
No motion & all calibration parameters \\
\hline
\end{tabular}
\vspace{-12pt}
\end{table}

\section{SIMULATION VALIDATION}
\label{sec:simulation}

The analytical results of Section~\ref{sec:observability} establish the
observable and unobservable directions of the proposed system. Simulation is
used here to verify that the estimator exhibits the predicted convergence and
degeneracy behaviour. A Monte-Carlo simulator generates IMU, camera, and
eight channel wheel encoder measurements together with ground truth
calibration parameters. The camera, IMU, and wheel-encoder rates are 10, 200, and 100 Hz, respectively. The MSCKF uses 11 clones and at most 100 features per update with \(\sigma_{px}=1\) px. Drive- and steering-encoder noise standard deviations are \(0.001\) rad/s and \(0.001\) rad, respectively, with lateral pseudo-noise \(\sigma_{\mathrm{lat}}=0.001\). For the perturbed initialization, wheel radii are offset by 1\%, wheel positions by 0.01 m, steering offsets and extrinsic orientation by 0.01 rad, extrinsic position by 0.1 m, and the temporal offset by 0.01 s. Six trajectories are considered:
the ICR-sweep excitation trajectory, straight line,
lateral, circular, stationary, and figure eight motion.
Following Section~\ref{sec:lateral_necessity}, $x_{w_1}$, $y_{w_1}$, and
$\delta_{o_1}$ are fixed to remove the odometry frame gauge freedom.

\subsection{Calibration and Degeneracy Validation}

Fig.~\ref{fig:exp1} shows the estimation errors and $3\sigma$ bounds for six
Monte-Carlo runs along the ICR-sweep trajectory. The small number of runs makes this figure ilustrative (not to be used as the primary statistical validation). Statistical consistency/accuracy is presented in Table~\ref{tab:degenerate_motions} (20 runs) All estimated calibration
parameters converge toward their ground truth values within the reported
bounds. The wheel radii converge faster initially than the wheel positions,
consistent with \eqref{eq:Cj}, since their sensitivities scale with
$\omega_{md_j}$ and $\omega$, respectively.
\begin{figure}[t]
    \centering
    \includegraphics[width=\columnwidth]
    {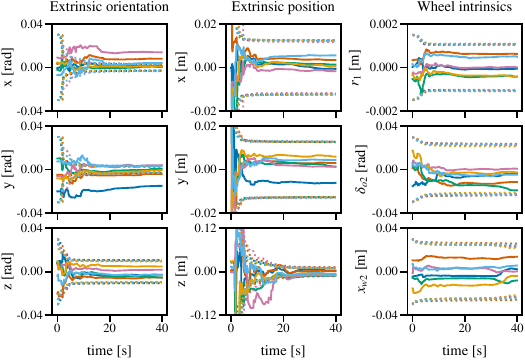}
    \caption{Estimation error (solid) and $3\sigma$ bound (dotted) for six
    Monte-Carlo runs along the ICR-sweep trajectory. Shown are the
    IMU-odometry extrinsics and representative wheel intrinsics
    ($r_1$, $\delta_{o_2}$, $x_{w_2}$).}
    \label{fig:exp1}
\end{figure}
The predicted degeneracies were also reproduced numerically. For
straight line and lateral motion ($\omega\equiv0$), the uncertainty of the
wheel position parameters does not contract, confirming
Lemma~\ref{lem:zero_yaw}. Under constant steering, the covariance projected
onto the two null directions predicted by the corresponding observability
analysis remains unchanged, while uncertainty in the orthogonal complement
decreases. A stationary platform leaves all calibration parameters
unobservable, as summarized in Table~\ref{tab:degenerate_motions}.
\subsection{Estimation Accuracy and Consistency}
Table~\ref{tab:rpe} reports relative pose error and mean NEES over
20 Monte-Carlo runs. With accurate initialization, enabling online
calibration has little effect. In contrast, incorrect fixed calibration
causes severe degradation, whereas online calibration recovers performance
close to the accurately initialized case. These results show that the proposed
calibration is particularly beneficial when the nominal kinematic parameters
are inaccurate.

With calibration enabled the pose NEES is $3.56$ and $3.72$ against a
reference value of $6$, i.e.\ mildly conservative, since a minimum noise floor
is added to the wheel relative pose covariance. Disabling calibration raises
it to $10.39$, the unmodelled parameter error not being represented in the
covariance; the value of $170.75$ reflects divergence in 5 of 20 runs
(median $2.03$) rather than a consistency level.

\begin{table}[t]
\centering
\small
\caption{Relative pose error (RPE) and mean NEES over 20 Monte-Carlo runs.}
\label{tab:rpe}
\resizebox{\columnwidth}{!}{%
\begin{tabular}{lrrrrr}
\toprule
Configuration &
RPE$_{10}$ [deg] &
RPE$_{10}$ [m] &
RPE$_{25}$ [deg] &
RPE$_{25}$ [m] &
NEES \\
\midrule
True init., calib. ON  & 0.014 & 0.0045 & 0.021 & 0.0055 & 3.56 \\
True init., calib. OFF & 0.050 & 0.0071 & 0.081 & 0.0111 & 10.39 \\
Bad init., calib. ON   & 0.014 & 0.0046 & 0.018 & 0.0060 & 3.72 \\
Bad init., calib. OFF  & 0.602 & 20.9823 & 1.465 & 47.4002 & 170.75 \\
\bottomrule
\end{tabular}%
}
\vspace{-12pt}
\end{table}

\section{EXPERIMENTAL VALIDATION}
The entire VIWO system was additionally validated on a real 4WIS4WID mobile robot. The robot is equipped with eight Dynamixel XH430-V210-T motors that have the option of synchronous command transmission and reading of velocity and position values. For VIWO, an Intel RealSense T265 was used, which has a stereo pair of fisheye cameras and an IMU, whose intrinsic and extrinsic parameters were calibrated using Kalibr~\cite{Furgale2013}. The ground truth position of the robot was measured using the OptiTrack system.

Figure~\ref{fig:experimental_results} shows the trajectories of wheel odometry, VIWO with calibration (VIWO w. calib), VIWO without calibration (VIWO wo. calib), and VIO. The root mean square errors of the listed trajectories are: VIWO w. calib 0.212/1.64, VIWO wo. calib 0.343/3.04, VIO 0.566/2.06, and wheel odometry 0.774/17.46 (meters/degrees). The initial values of the robot parameters, as well as the final ones, can be seen in the Table~\ref{tab:calibration_table}. The parameters converged with respect to the initial values by a small amount, but such a change can have a much greater influence over longer distances. The influence of the parameters is visible from the graphs and the RMS error values.
\begin{figure}[t]
  \centering
  \includegraphics[width=\columnwidth]{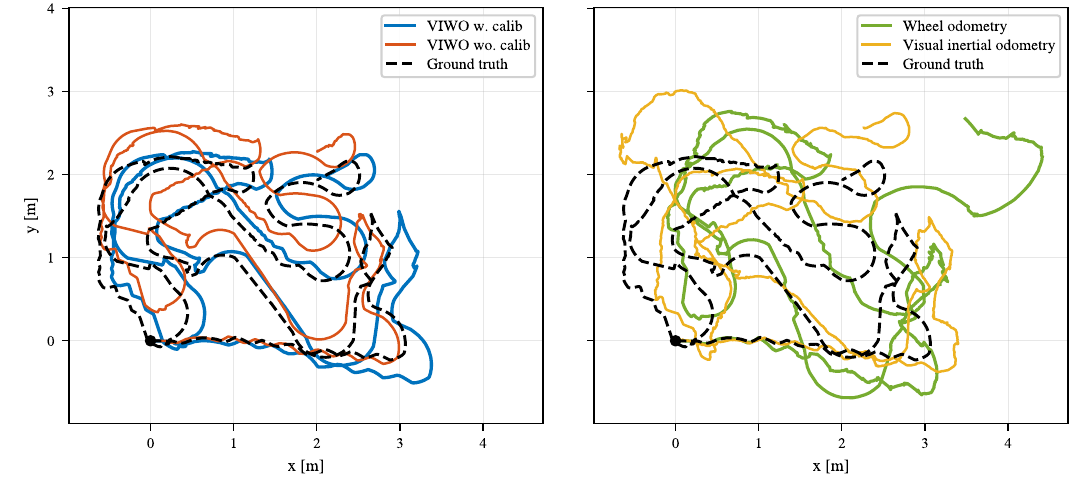}
  \caption{Experimental results}
  \label{fig:experimental_results}
\end{figure}

\begin{table}[t]
  \centering
  \caption{Values of the calibration parameters before and after online calibration for real world experiment}
  \label{tab:calibration_table}
  \resizebox{\columnwidth}{!}{%
  \begin{tabular}{lcc}
    \toprule
    \textbf{Parameter} & \textbf{Before} & \textbf{After} \\
    \midrule
    wheel radii $r$ [mm] & [25.6, 25.6, 25.6, 25.6] & [26.3, 26.3, 28.1, 26.8] \\
    wheel positions $x_w$ [m] & [-0.1125, -0.1125, 0.1125] & [-0.1029, -0.1029, 0.1454] \\
    wheel positions $y_w$ [m] & [0.1125, -0.1125, -0.1125] & [0.1267, -0.1216, -0.1216] \\
    steering offsets $\delta_o$ [rad] & [0.0, 0.0, 0.0] & [0.0168, -0.0098, -0.0633] \\
    Ext. Pos [m] & [0.1457, 0.0279, 0.0576] & [0.1696, 0.0192, 0.0244] \\
    Ext. Ori [rad] & [1.2085, 1.2093, 1.2091] & [1.2498, 1.2060, 1.2206] \\
    Time offset [s] & 0.00000 & -0.00656 \\
    \bottomrule
  \end{tabular}}
  \vspace{-12pt}
\end{table}

\section{CONCLUSION AND FUTURE WORK}
\label{sec:conclusion}
In this paper, we developed a tightly coupled MSCKF based VIWO system for 4WIS4WID mobile robots that fuses camera, IMU, and wheel encoder measurements while performing online calibration. The estimator simultaneously estimates the spatiotemporal odometry extrinsics and thirteen kinematic parameters using a redundant 2D wheel odometry model using longitudinal rolling and lateral no-slip constraints. Observability analysis showed that a drive only model leaves all steering offsets unobservable, while the proposed model restores their observability and increases measurement redundancy. In addition to the standard VINS unobservable directions, three motion independent directions related to the arbitrary $SE(2)$ odometry frame placement were identified, together with several degenerate motions and the excitation conditions required for calibration. The system was validated in simulation (on 20 Monte Carlo runs) and real world experiment on a 4WIS4WID robot.

Future work will focus on validation on vertical surfaces and over longer travelled distances, where calibration errors are expected to have a stronger effect. We also plan to incorporate robot dynamics to account for wheel slip and violations of the no-slip assumptions used by the current kinematic model.
\section*{ACKNOWLEDGMENT}
The authors acknowledge the support of the VANGUARD project, funded by the European Union under the Competitiveness and Cohesion Programme 2021–2027.

\bibliographystyle{IEEEtran}
\bibliography{refrences}

\end{document}